%% file: main.tex
\documentclass[twoside,11pt]{article} 
\usepackage{jmlr2e}
\usepackage{graphicx}
\usepackage{amsfonts}
\usepackage{braket}
\usepackage{boxedminipage}
\usepackage{epsf}
\usepackage{algorithm}
\usepackage{bm}
\usepackage{amsmath}
\usepackage{xspace}
\usepackage{wrapfig}
 \usepackage{url}
\def\indot<#1>{\langle #1 \rangle}

\author{\name Daiki Suehiro
 \email{suehiro@ait.kyushu-u.ac.jp} \\
 \addr Faculty of Information Science and Electrical Engineering, Kyushu University\\
and AIP, RIKEN\\
 Fukuoka, Japan, 8190395
\AND
\name Kohei Hatano
 \email{hatano@inf.kyushu-u.ac.jp}\\
 \addr Research and Development Division, Library,
 Kyushu University \\
and AIP, RIKEN\\
 Fukuoka, Japan, 8190395
\AND
 \name Eiji Takimoto
 \email{eiji@inf.kyushu-u.ac.jp}\\
 \addr Department of Informatics,  Kyushu University\\
 Fukuoka, Japan, 8190395
\AND
\name Shuji Yamamoto
 \email{yamashu@math.keio.ac.jp}
\AND
\name Kenichi Bannai
 \email{bannai@math.keio.ac.jp}\\
 \addr Department of Mathematics, Keio University\\
and AIP, RIKEN\\
 Kanagawa, Japan, 2238522
\AND
\name Akiko Takeda
 \email{atakeda@ism.ac.jp}\\
 \addr Department of Mathematical Analysis and Statistical Inference,\\ 
The Institute of Statistical Mathematics\\ and AIP, RIKEN\\
 Tokyo, Japan, 1908562
}

\editor{}

\begin{document}
\title{Boosting the kernelized shapelets: \\Theory and algorithms for local features}
\input{local.tex}

\input{symbol.tex}
\maketitle

\begin{abstract}
We consider binary classification problems using 
local features of objects.  One of motivating applications is
 time-series classification, where features reflecting some local
 closeness measure between
a time series and a pattern sequence called shapelet are useful.
 Despite the empirical success of such approaches using local features,
 the generalization ability of resulting hypotheses is not fully
 understood and previous work relies on a bunch of heuristics. 
 In this paper, we formulate a class of hypotheses using local features,
 where the richness of features is controlled by kernels.
 We derive generalization bounds of sparse ensembles over the class
 which is exponentially better than a standard analysis in terms of the number of possible local features. 
The resulting optimization problem is well suited to the boosting
 approach and the weak learning problem is formulated as a DC program,
 for which practical algorithms exist.  
In preliminary experiments on time-series data sets, our method
achieves competitive accuracy with the state-of-the-art algorithms
with small parameter-tuning cost.
\end{abstract}
\input{intro}
\input{prelim}
\input{theorem1}

\input{algorithm}
\input{experiments}
\input{conclusion}
\bibliography{ref,similarity,hatano}
\input{appendix}
\end{document}

%% file: local.tex

\newtheorem{defi}{Definition}
\newtheorem{theo}{Theorem}
\newtheorem{prop}[theo]{Proposition}
\newtheorem{coro}[theo]{Corollary}
\newtheorem{lemm}{Lemma}
\newtheorem{fact}{Fact}
\newtheorem{ex}{Example}
\def\remark{\par\noindent\hangindent0pt{\bf Remark.}~}


\def\OMIT#1{}
\def\newwd#1{{\em #1}}

\newcommand{\mnote}[1]{\marginpar{#1}}
\newcommand{\mynote}[1]{{\bf {#1}}}

%% file: symbol.tex
%
%

\newcounter{nombre}
\renewcommand{\thenombre}{\arabic{nombre}}
\setcounter{nombre}{0}
\newenvironment{OP}[1][]{\refstepcounter{nombre}\par\bigskip \abovedisplayskip=0.5\abovedisplayskip \noindent{\sf OP \thenombre : #1}}{}

\global\long\def\T#1{#1^{\top}}

\newcommand{\shapelets}{shapelets}
\newcommand{\shapelet}{shapelet}
\newcommand{\LPS}[0]{LPM\xspace}
\newcommand{\targetH}{{local pattern matching hypothesis}}
\newcommand{\Nh}{N_{\mathrm{high}}}
\newcommand{\fsvm}{f_{\mathrm{SVM}}}
\newcommand{\frsvm}{f_{\mathrm{RSVM}}}
\newcommand{\OPTH}{\mathrm{OPT}_\textrm{hard}}
\newcommand{\OPTS}{\mathrm{OPT}_\textrm{soft}}
\newcommand{\OPTSNU}{\mathrm{OPT}_\mathrm{soft}(\nu^+)}
\newenvironment{claim}{\begin{trivlist}\item[]\textit{Claim}}{\end{trivlist}}
\newcommand{\AUC}{\mathrm{AUC}}
\newcommand{\bphi}{{\bf \phi}}
\newcommand{\bx}{{\bf x}}
\newcommand{\bsigma}{\boldsymbol{\sigma}}
\newcommand{\blambda}{\boldsymbol{\lambda}}
\newcommand{\kernel}{\boldsymbol{\mathrm{K}}}
\newcommand{\Vmat}{\boldsymbol{\mathrm{V}}}
\newcommand{\Xmat}{\boldsymbol{\mathrm{X}}}
\newcommand{\bw}{{\bf w}}
\newcommand{\bd}{{\bf d}}
\newcommand{\bv}{{\bf v}}
\newcommand{\bu}{{\bf u}}
\newcommand{\bz}{{\bf z}}
\newcommand{\bs}{{\bf s}}
\newcommand{\bt}{{\bf t}}
\newcommand{\balpha}{\boldsymbol{\alpha}}
\newcommand{\bbeta}{\boldsymbol{\beta}}
\newcommand{\bgamma}{\boldsymbol{\gamma}}
\newcommand{\bxi}{\boldsymbol{\xi}}
\newcommand{\sbsq}{\mathrm{sub}}
\newcommand{\edge}{\mathrm{edge}}
\newcommand{\Ksub}{K_{\mathrm{sub}}}
\newcommand{\Ourmethod}[0]{our method\xspace}
\newcommand{\hsigma}{\widehat\sigma}
\newcommand{\convhull}{\mathcal{H}}
\newcommand{\err}{\mathrm{err}}
\newcommand{\RW}{\mathrm{RW}}
\newcommand{\RCS}{\mathrm{RCS}}
\newcommand{\RV}{\mathrm{RV}}
\newcommand{\SRS}{\mathrm{SRS}}
\newcommand{\RSG}{\mathrm{RSG}}
\newcommand{\sign}{\mathrm{sign}}
\newcommand{\dom}{\mathcal{X}} 
\newcommand{\domp}{\mathcal{X}^{pos}} 
\newcommand{\domn}{\mathcal{X}^{neg}} %
\newcommand{\range}{\mathcal{Y}} 
\newcommand{\Natural}{\mathbb{N}} %
\newcommand{\Real}{\mathbb{R}} 
\newcommand{\Hilbert}{\mathbb{H}} 
\newcommand{\Prob}{\mathcal{P}}
\newcommand{\F}{\mathcal{F}}
\newcommand{\calS}{\mathcal{S}}
\newcommand{\calF}{\mathcal{F}}
\newcommand{\calG}{\mathcal{G}}
\newcommand{\calT}{\mathcal{T}}
\newcommand{\calY}{\mathcal{Y}}
\newcommand{\Hyp}{\mathcal{H}}
\newcommand{\WH}{\mathcal{W}}
\newcommand{\EX}{\mathrm{EX}}
\newcommand{\filtEX}{\mathrm{FiltEX}}
\newcommand{\HSelect}{\mathrm{HSelect}}
\newcommand{\WL}{\mathrm{WL}}
\newcommand{\vecx}{\boldsymbol{x}}
\newcommand{\vecy}{\mbox{\boldmath $y$}}
\newcommand{\vecw}{\boldsymbol{w}}
\newcommand{\vecz}{\boldsymbol{z}}
\newcommand{\vecg}{\mbox{\boldmath $g$}}
\newcommand{\veca}{\mbox{\boldmath $a$}}
\newcommand{\vecd}{\boldsymbol{d}}
\newcommand{\vecell}{\mbox{\boldmath $\ell$}}
\newcommand{\vecsigma}{\boldsymbol{\sigma}}
\newcommand{\vecpi}{\boldsymbol{\pi}}
\newcommand{\vecxi}{\boldsymbol{\xi}}
\newcommand{\vece}{\mbox{\boldmath $e$}}
\newcommand{\vecB}{\mbox{\boldmath $B$}}
\newcommand{\vecD}{\mbox{\boldmath $D$}}
\newcommand{\vecI}{\mbox{\boldmath $I$}}
\newcommand{\tr}{\mathrm{tr}}
\newcommand{\vecG}{\mbox{\boldmath $G$}}
\newcommand{\vecF}{\mbox{\boldmath $F$}}
\newcommand{\tvecu}{\tilde{\mbox{\boldmath $u$}}}
\newcommand{\tvecw}{\tilde{\mbox{\boldmath $w$}}}
\newcommand{\tvecx}{\tilde{\mbox{\boldmath $x$}}}
\newcommand{\tw}{\tilde{w}}
\newcommand{\tx}{\tilde{x}}
\newcommand{\haty}{\hat{y}}
\newcommand{\vecf}{\mbox{\boldmath $f$}}
\newcommand{\vectheta}{\mbox{\boldmath $\theta$}}
\newcommand{\vecalpha}{\boldsymbol{\alpha}}
\newcommand{\vecbeta}{\mbox{\boldmath $\beta$}}
\newcommand{\vectildealpha}{\widetilde{\vecalpha}}
\newcommand{\vectildebeta}{\widetilde{\vecbeta}}
\newcommand{\tildealpha}{\widetilde{\alpha}}
\newcommand{\tildebeta}{\widetilde{\beta}}
\newcommand{\vechatalpha}{\widehat{\vecalpha}}
\newcommand{\vechatbeta}{\widehat{\vecbeta}}
\newcommand{\hatalpha}{\widehat{\alpha}}
\newcommand{\hatbeta}{\widehat{\beta}}
\newcommand{\vectau}{\mbox{\boldmath $\tau$}}
\newcommand{\veclambda}{\bm{\lambda}}
\newcommand{\vecu}{\mbox{\boldmath $u$}}
\newcommand{\vecv}{\mbox{\boldmath $v$}}
\newcommand{\vecp}{\boldsymbol{p}}
\newcommand{\vecq}{\mbox{\boldmath $q$}}
\newcommand{\vecr}{\boldsymbol{r}}
\newcommand{\vecc}{\boldsymbol{c}}
\newcommand{\fp}{\mathrm{fp}}
\newcommand{\fn}{\mathrm{fn}}
\newcommand{\ouralg}{{Our algorithm}~}
\newcommand{\Ouralg}{PUMMA~}

\newcommand{\bn}{\Delta_2} 
\newcommand{\psimp}{\mathcal{P}} %
\newcommand{\hatgamma}{\hat{\gamma}}
\newcommand{\indctr}[1]{I(#1)}
\newcommand{\CLASS}{\mathcal{C}}
\newcommand{\VC}{\mathrm{VC}}

\newcommand{\reg}{\mathcal{R}}
\newcommand{\breg}{D}

\newcommand{\filtex}{\mathrm{GenD_t}}
\newcommand{\gensamp}{\mathrm{GenSample}}

\newcommand{\argmax}{\mathop{\rm arg~max}\limits}
\newcommand{\argmin}{\mathop{\rm arg~min}\limits}
\newcommand{\Expo}{\mathop{\rm  E}\limits}

\newcommand{\half}{\frac{1}{2}}
\newcommand{\eps}{\varepsilon}

\newcommand{\hp}{\hat{p}}
\newcommand{\hmup}{\hat{\mu}[+]}
\newcommand{\hmun}{\hat{\mu}[-]}
\newcommand{\hgp}{\hat{\gamma}[+]}
\newcommand{\hgn}{\hat{\gamma}[-]}
\newcommand{\gain}{\Delta}
\newcommand{\hgain}{\hat{\Delta}}
\newcommand{\vecdelta}{\boldsymbol{\delta}}

\newcommand{\tildeO}{\Tilde{O}}
\newcommand{\permset}{S}
\newcommand{\base}{\boldsymbol{B}}
\newcommand{\calC}{\mathcal{C}}
\newcommand{\calP}{\mathcal{P}}
\newcommand{\calX}{\mathcal{X}}
\newcommand{\calE}{\mathcal{E}}
\newcommand{\Rdm}{\mathfrak{R}}
\newcommand{\GC}{\mathfrak{G}}
\newcommand{\hullC}{\mathrm{conv}(\calC)}
\newcommand{\hullH}{\mathrm{conv}(H)}
\newcommand{\conv}{\mathrm{conv}}

\newcommand{\Rmin}{R_{\mathrm{min}}}
\newcommand{\Rmax}{R_{\mathrm{max}}}

\def\ceil#1{%
\left\lceil #1 \right\rceil}

\def\defeq{%
\stackrel{\mathrm{def}}{=}}

\def\floor#1{%
\lfloor #1 \rfloor}

\def\myhang{%
    \par\noindent\hangindent20pt\hskip20pt}
\def\nitem#1{%
    \par\noindent\hangindent40pt
    \hskip40pt\llap{#1~}}

\newcommand{\E}{\boldsymbol{E}}

\newcommand{\pdiff}{\Phi_{\mathrm{diff}}(\calP_S)}
\newcommand{\calI}{\mathcal{I}}

%% file: intro.tex
\section{Introduction}
Classifying objects using their ``local'' patterns is often effective in
various applications.
For example, in time-series classification problems, 
a local feature called \textit{shapelet} is shown to be quite powerful
in the data mining literature~\citep{Ye:2009:TSS:1557019.1557122,
  KeoghR13, Hills:2014:CTS:2597434.2597448,
  Grabocka:2014:LTS:2623330.2623613}.
More precisely, a shapelet $\bz=(z_1,\dots,z_\ell)$ is a real-valued 
``short'' sequence in $\Real^\ell$ for some $\ell>1$.
 Given a time-series $\bx=(x_1,\dots,x_L)$,
 a typical measure of closeness between the time-series $\bx$ and the
 shapelet $\bz$ is $\min_{j=1}^Q \|\bx_{j:j+\ell-1}-\bz\|_2$, where
 $Q=L-\ell +1$ and $\bx_{j:j+\ell-1}=(x_{j},\dots,x_{j+\ell-1})$.
 Here, the measure focuses on ``local'' closeness between the
 time-series and the shapelet.
 In many time-series classification problems, sparse combinations of
 features based on the closeness to some shapelets are
 useful~\citep{GrabockaWS15, renard:hal-01217435, HouKZ16}. 
 Similar situations could happen in other applications. Say,
 for image classification problems, {\it template matching} is a well-known
 technique to measure similarity between an image and  ``(small)
 template image'' in a local sense.

 Despite the empirical success of applying local features, theoretical
 guarantees of such approaches are not fully investigated.
 In particular, trade-offs of the richness of such local features and
 the generalization ability are not characterized yet. 
 
In this paper, we formalize a class of hypotheses based on some local
closeness. Here, the richness of the class is controlled by associated
kernels. We show generalization bounds of ensembles of such local
classifiers. Our bounds are exponentially tighter in terms of some
parameter  than typical bounds obtained by a standard analysis.

Further, for learning ensembles of the kernelized hypotheses,
our theoretical analysis suggests a $1$-norm constrained
optimization problem with infinitely many parameters, for which the dual
problem is categorized as a semi-infinite program~\citep[see][]{Shapiro09}.
To obtain approximate solutions of the problem efficiently,
we take the approach of boosting~\citep{schapire-etal:as98}. In particular, we employ
LPBoost~\citep{demiriz-etal:ml02}, which solve $1$-norm constrained soft margin
optimization via a column generation approach.
 As a result, our approach has two stages, where the master problem is a
 linear program and the sub-problems are difference of convex programs
 (DC programs),
 which are non-convex.
 While it is difficult to solve the sub-problems exactly due to
 non-convexity, various techniques are investigated for DC programs and
 we can find good approximate solutions efficiently for many cases in practice.

 In preliminary experiments on time-series data sets,
our method achieves competitive accuracy with the state-of-the-art
algorithms using shapelets.
While the previous algorithms need careful parameter tuning and heuristics,
our method uses less parameter tuning and parameters can be determined
in an organized way.
In addition, our solutions tend to be sparse and could be useful for
domain experts to select good local features.

\subsection{Related work}
The concept of time-series shapelets was first introduced 
by~\citet{Ye:2009:TSS:1557019.1557122}.
The algorithm finds shapelets by using the information gains
of potential candidates associated with all the subsequences of the
given time series and constructs a decision tree.
{\it Shapelet transform}~\citep{Hills:2014:CTS:2597434.2597448} 
is a technique combining with shapelets and machine learning.
The authors consider the time-series examples as feature vectors
defined by the set of local closeness to some shapelets 
and in order to obtain classification rule, they employed 
some effective learning algorithms such as linear SVM or random forests.
Note that shapelet transform 
completely separate the phase searching for shapelets 
from the phases of creating classification rules.
Afterward, 
many algorithms have been proposed
to search the good shapelets efficiently
keeping high prediction accuracy
in practice~\citep{KeoghR13, GrabockaWS15, 
renard:hal-01217435, Karlsson:2016}.
The algorithms are based on the idea that 
discriminative shapelets
are contained in the training data. 
This approach, however, might overfit without a regularization.
{\it Learning Time-Series Shapelets} (LTS)
algorithm \citep{Grabocka:2014:LTS:2623330.2623613} is a 
different approach from such subsequence-based algorithms.
LTS approximately solves an optimization problem of learning
the best shapelets directly without searching subsequences in a
brute-force way.
In contrast to the above subsequence-based methods, 
LTS finds nearly optimal shapelets and achieves 
higher prediction accuracy than the other existing methods in
practice.
However, there is no theoretical guarantee of its generalization error.

There are previous results using 
local features based on kernels ~\citep[e.g.,][]{OdoneBV05, harris54,
  Shimodaira:2001, 5597650}.
However, their approaches focus on the closeness (similarity) between 
examples (e.g., subsequence $a$ of example $A$ and subsequence $b$ of
example $B$), not known to capture local similarity.

%% file: prelim.tex
\def\calP{\mathcal{P}}
\section{Preliminaries}
\label{sec:prelim}
Let $\calP \subseteq \Real^\ell$ be a set, in which an element is called a pattern.
Our instance space $\calX$ is a set of sequences of 
patterns in $\calP$.
For simplicity, we assume that every sequence in $\calX$ is of
the same length.
That is, $\calX \subseteq \calP^Q$ for some integer $Q$.
We denote by $\bx = (\bx^{(1)}, \ldots, \bx^{(Q)})$ an instance
sequence in $\calX$,
where every $\bx^{(j)}$ is a pattern in $\calP$.
The learner receives a labeled sample
$S = (((\bx_1^{(1)}, \ldots, \bx_1^{(Q)}), y_1), \ldots, 
((\bx_m^{(1)}, \ldots, \bx_m^{(Q)}), y_m)) \in
(\calX \times \{-1, 1\})^m$ of size $m$, where 
each labeled instance is independently drawn according to some unknown
distribution $D$ over $\calX \times \{-1,+1\}$.

Let $K$ be a kernel over $\calP$, which is used to measure the
similarity between patterns, and let
$\Phi: \calP \to \Hilbert$ denote a feature map associated with
the kernel $K$ for a Hilbert space $\Hilbert$.
That is,
$K(\bz, \bz')=\langle \Phi(\bz), \Phi(\bz')\rangle$
for patterns $\bz, \bz' \in \calP$, where
$\langle \cdot,\cdot \rangle$ denotes the inner product over $\Hilbert$.
The norm induced by the inner product is denoted by
$\|\cdot\|_\Hilbert$ and satisfies $\|\bu\|_{\Hilbert} =
\sqrt{\langle \bu, \bu \rangle}$ for $\bu \in \Hilbert$. 

For each $\bu \in \Hilbert$,
we define the {\it base classifier} (or the {\it feature}),
denoted by $h_{\bu}$, 
as the function that maps a given sequence
$\bx = (\bx^{(1)},\ldots,\bx^{(Q)}) \in \calX$ to the maximum of
the similarity scores between $\bu$ and $\bx^{(j)}$
over all patterns $\bx^{(j)}$ in $\bx$.
More specifically,
\[
	h_{\bu}(\bx) = \max_{j \in [Q]}
		\left\langle\bu, \Phi(\bx^{(j)}) \right\rangle,
\]
where $[Q]$ denotes the set $\{1,2,\ldots,Q\}$.
For a set $U \subseteq \Hilbert$, we define the class of base classifiers
as 
\[
	H_U = \left\{ h_{\bu} \mid \bu \in U \right\}
\]
and we denote by $\conv(H_U)$ the set of convex combinations of
base classifiers in $H_U$. More precisely,
\[
	 \conv(H_U) = \left\{
		\sum_{\bu \in U'} w_\bu h_\bu(\bx) \mid
		\forall \bu \in U', w_\bu \geq 0,
		\sum_{\bu \in U'} w_\bu = 1,
		\text{$U' \subseteq U$ is a finite support} 
 	\right\}.
\]
The goal of the learner is to find a final hypothesis
$g \in \conv(H_U)$, so that its generalization error
$\calE_D(g) = \Pr_{(\bx,y) \sim D}[\sign(g(\bx)) \neq y]$
is small.

\paragraph{Example: Learning with time-series shapelets}
For a typical setting of learning with time-series shapelets,
an instance is a sequence of real numbers 
$\bx = (x_1,x_2,\ldots,x_L) \in \Real^L$ and a base classifier
$h_\bs$, which is associated with a shapelet
(i.e., a ``short'' sequence of real numbers)
$\bs = (s_1,s_2,\ldots,s_\ell) \in \Real^\ell$, is defined as
\[
	h_\bs(\bx) = \max_{1 \leq j \leq L-\ell+1}
	K(\bs, \bx^{(j)}),
\]
where $\bx^{(j)} = (x_j,x_{j+1},\ldots,x_{j+\ell})$
is the subsequence of $\bx$ of length $\ell$ that begins
with $j$th index\footnote{
In previous work, $K$ is not necessarily a kernel. For instance,
the negative of Euclidean distance
$-\|\bs - \bx^{(j)}\|$ is often used as a similarity measure $K$.}.
In our framework, this corresponds to the case where
$Q = L-\ell+1$, $\calP \subseteq \Real^\ell$, and
$U = \{\Phi(\bs) \mid \bs \in \calP\}$.

%% file: theorem1.tex
\section{Risk bounds of the hypothesis classes}
\label{sec:theorem1}
In this section, we give generalization bounds of
the hypothesis classes $\conv(H_U)$ for various $U$ and $K$.
To derive the bounds, we use the Rademacher and the Gaussian
complexity~\citep{Bartlett:2003:RGC}.
\begin{defi}{~\citep[The Rademacher and the Gaussian complexity,][]{Bartlett:2003:RGC}}
Given a sample $S=(\bx_1,\dots,\bx_m) \in \calX^m$, 
the empirical Rademacher complexity $\Rdm(H)$ of a class $H \subset
 \{h: \calX \to \Real\}$ w.r.t.~$S$
 is defined as 
 $\Rdm_S(H)=\frac{1}{m}\Expo_{\vecsigma}\left[
\sup_{h \in H}\sum_{i=1}^m \sigma_i h(\bx_i)
 \right]$,
 where $\vecsigma \in \{-1,1\}^m$ and each $\sigma_i$ is an independent
 uniform random variable in $\{-1,1\}$.
 The empirical Gaussian complexity $\GC_S(H)$ of $H$ w.r.t.~$S$
 is defined similarly but
 each $\sigma_i$ is drawn independently from the standard normal distribution.
\end{defi}
The following bounds are well-known. 
\begin{lemm}{\citep[Lemma 4]{Bartlett:2003:RGC}}
 \label{lemm:RC_and_GC}
 $\Rdm_S(H) =O(\GC_S(H))$.
\end{lemm}

\begin{lemm}{\citep[Corollary 6.1]{Mohri.et.al_FML}} 
\label{lemm:ensemble_margin_bound}
For fixed $\rho$, $\delta >0$, 
the following bound holds with probability at least $1- \delta$:
for all $f \in \hullH$,
\[
\calE_D(f) \leq \calE_{S,\rho}(f) + \frac{2}{\rho} \Rdm_S(H) + 3 \sqrt{\frac{\log\frac{1}{\delta}}{2m}},
\]
 where $\calE_{S, \rho}(f)$ is the \textit{empirical margin loss} of $f$
 over $S$,
 i.e., the ratio of examples for which $f$ has margin $y_i f(\bx_i)<\rho$.
\end{lemm}

To derive generalization bounds based on the Rademacher or the Gaussian
complexity is quite standard in the statistical learning theory
literature and applicable to our classes of interest as well.
However, a standard analysis provides us sub-optimal bounds. 

For example, let us consider the simple case where the class $H_U$ of
base classifiers is defined by the linear kernel with $U$ to be
the set of vectors in $\Real^\ell$ of bounded norm.
In this case, $H_U$ can be viewed as 
$H_U = \{\max\{h_1,\dots,h_Q\} \mid h_1 \in H_1, \ldots, h_Q \in H_Q \}$,
where
$H_j = \{h: \bx \mapsto \langle \bu, \bx^{(j)} \rangle \mid \bu \in U\}$
for every $j \in [Q]$.
Then, by a standard analysis ~\citep[see, e.g.,][Lemma 8.1]{Mohri.et.al_FML}, 
we have $\Rdm(H_U) \leq \sum_{j=1}^Q \Rdm(H_j) =O \left(\frac{Q}{\sqrt{m}} \right)$.
However, this bound is weak since it is linear in $Q$.
In the following, we will give an improved bound of
$\tilde{O}(\log Q/\sqrt{m})$.
The key observation is that if base classes $H_1,\dots,H_Q$ are
``correlated'' somehow, one could obtain better Rademacher bounds.
 In fact, we will exploit some geometric properties among these base classes.

\subsection{Main theorem}
First, we show our main theoretical result on the generalization 
bound of $\conv(H_U)$.

Given a sample $S$, let $\calP_S$ be the set of patterns appearing in
$S$ and 
let $\Phi(\calP_S)=\{\Phi(\bz) \mid \bz \in \calP_S\}$.
Let $\pdiff=\{\Phi(\bz) - \Phi(\bz') \mid
\bz,\bz'\in \calP_S, \bz \neq \bz'\}$.
By viewing each instance $\bv \in \pdiff$ as a hyperplane
$\{\bu \mid \langle \bv, \bu \rangle = 0\}$,
we can naturally define a partition
of the Hilbert space $\Hilbert$
by the set of all hyperplanes $\bv \in \pdiff$.
Let $\calI$ be the set of all cells of the partition.
Each cell $I \in \calI$ is a polyhedron which is defined by
a minimal set $V_I \subseteq \pdiff$ that
satisfies $I = \bigcap_{\bv \in V_I}
\{\bu \mid \langle \bu, \bv \rangle \geq 0\}$.
Let
\[
	\mu^* = \min_{I \in \calI}\max_{\bu\in I \cap U}\min_{\bv \in V_I}
 |\langle \bu, \bv \rangle|. 
\]
Let $d^*_{\Phi,S}$ be the VC dimension of the set of linear classifiers
over the finite set $\pdiff$, given by
$F_U=\{ f: \bv \mapsto \sign(\langle \bu, \bv\rangle) \mid \bu \in U\}$. 
Then, our main result is stated as follows:

\begin{theo}
\label{theo:main}
Suppose that for any $\bz \in \calP$, $\|\Phi(\bz)\|_\Hilbert \leq R$.
Then, for any $\rho>0$ and $\delta (0 <\delta <1)$, 
with probability at least $1-\delta$, 
the following holds
for any $g \in \conv(H_U)$ with
$U \subseteq \{\bu \in \Hilbert \mid \|\bu\|_\Hilbert \leq \Lambda\}$:
\begin{align}
 \calE_D(g) \leq& \calE_{\rho}(g)
 +O\left(
 \frac{R\Lambda \sqrt{d_{\Phi, S}^* \log (mQ)}}{\rho\sqrt{m}}  + \sqrt{\frac{\log\frac{1}{\delta}}{m}}
\right).
\end{align}
  In particular,
(i) if $\Phi$ is the identity mapping (i.e., the associated kernel
 is the linear kernel), or
(ii) if $\Phi$ satisfies that
$\left\langle\Phi(\bz), \Phi(\bx) \right\rangle$ is monotone
decreasing with respect to $\|\bz-\bx\|_2$  (e.g.,
the mapping defined by the Gaussian kernel) and
$U=\{\Phi(\bs) \mid \bs \in \Real^\ell, \|\Phi(\bs)\|_\Hilbert \leq \Lambda\}$, 
then $d^*_{\Phi, S}$ can be replaced by $\ell$.
(iii) Otherwise, $d^*_{\Phi,S} \leq R\Lambda/\mu^*$.
\end{theo}
In order to prove the theorem~\ref{theo:main}, we show several 
definitions, lemmas and the proofs as following subsection.
\subsection{Proof sketch}

\begin{defi}{The set $\Theta$ of mappings from an instance to a pattern}
For any $\bu \in U$, let $\theta_\bu: [m] \to [Q]$ be a
mapping defined by
\[
\theta_{\bu}(i) := \arg\max_{j \in [Q]} 
 \left\langle \bu,  \Phi\left(\bx_i^{(j)}\right)  \right\rangle,
\]
and we denote the set of all $\theta_{\bu}$ as 
$\Theta = \{\theta_{\bu} \mid \bu \in U \}$.
\end{defi}

\begin{lemm}
\label{lemm:GC}
Suppose that for any $\bz \in \calP$, $\|\Phi(\bz)\|_\Hilbert \leq R$.
Then, the empirical Gaussian complexity of $H_U$ with respect to $S$
for $U \subseteq \{\bu \mid \|\bu\|_\Hilbert \leq \Lambda\}$
is bounded as follows: 
\[
\GC_{S}(H) \leq \frac{R\Lambda\sqrt{(\sqrt{2}-1)+ 2(\ln|\Theta|)}}{\sqrt{m}}.
\]
\end{lemm}
The proof is given in the supplemental material.

Thus, it suffices to bound the size $|\Theta|$. 
Naively the size $|\Theta|$ is at most $Q^m$ since there are $Q^m$
possible mappings from $[m]$ to $[Q]$. However, this naive bound is too
pessimistic.
The basic idea to get better bounds is the following.
Fix any $i \in [m]$ and consider points $\Phi(\bx_i^{1}),
\dots,\Phi(\bx_i^{Q})$.
Then, we define equivalence classes of $\bu$ such that
$\theta_\bu(i)$ is the same, which define a Voronoi diagram
for the points $\Phi(\bx_i^{1}),\dots,\Phi(\bx_i^{Q})$.
Note here that the closeness is measured by the inner product, not a
distance. More precisely, 
let $V_i=(V_i^{(1)},\dots,V_i^{(Q)})$ be the Voronoi diagram
defined as $V_i^{(j)}=\{\bu \in \Hilbert \mid \theta_\bu(i) = j\}$.
Let us consider the set of intersections
$\cap_{i\in [m]}V_i^{(j_i)}$ for all
combinations of $(j_1,\dots, j_m) \in [Q]^m$.
The key observation is that each non-empty intersection corresponds to a
mapping $\theta \in \Theta$. Thus,
we obtain $|\Theta|=(\text{the number of intersections $\cap_{i\in
[m]}V_{i}^{(j_i)}$})$. In other words, the size of $\Theta$ is exactly
the number of rooms defined by the intersections of $m$ Voronoi
 diagrams $V_1,\dots,V_m$.
 From now on, we will derive upper bounds based on this observation.
 
 \begin{lemm}
  \label{lemm:Theta}
 \[
  |\Theta| =O((mQ)^{2d_{\Phi,S}^*}).
 \]
 \end{lemm}
The proof is shown in the supplemental material.

  \begin{theo}\noindent \mbox\\
   \label{theo:Theta} 
  \begin{enumerate}
   \item[(i)] If $\Phi$ is the identity mapping over $\calP$, then
	      $|\Theta|=O((mQ)^{2\ell})$.
   \item[(ii)] if $\Phi$ satisfies that $\left\langle\Phi(\bz), \Phi(\bx) \right\rangle$ is monotone
decreasing with respect to $\|\bz-\bx\|_2$  (e.g.,
	the mapping defined by the Gaussian kernel) and
 $U=\{\Phi(\bs) \mid \bs \in \Real^\ell, \|\Phi(\bs)\|_\Hilbert \leq \Lambda\}$, 
then $|\Theta|=O((mQ)^{2\ell})$.
   \item[(iii)] Otherwise,
	         $|\Theta|=O((mQ)^{R\Lambda/\mu^*})$.
  \end{enumerate}
  \end{theo}
The proof is shown in the supplemental material.

Now we are ready to prove Theorem~\ref{theo:main}.
\begin{proof}[Proof of Theorem~\ref{theo:main}]
By using Lemma~\ref{lemm:RC_and_GC}, and
 \ref{lemm:ensemble_margin_bound},
we obtain the generalization bound in terms of the Gaussian
 complexity of $H$. Then, by applying Lemma~\ref{lemm:GC} and Theorem~\ref{theo:Theta},
we complete the proof.
\end{proof}

%% file: algorithm.tex
\section{Optimization problem formulation}
In this section, we formulate an optimization problem to learn ensembles
in $\conv(H_U)$ for $U \subseteq \{ \bu: \|\bu\|_\Hilbert \leq \Lambda\}$.
The problem is a $1$-norm constrained soft margin optimization problem
using hypotheses in $H_U$, which is a linear program.

\begin{align}\label{align:LPBoostPrimal}
\max\limits_{\rho,\bw,\bxi} \quad &
 \rho  - \frac{1}{\nu m}\sum_{i=1}^m \xi_{i}
\\ \nonumber
 \text{sub.to} \quad 
& \int_{\bu \in U}y_i \bw_{\bu}h_\bu(\bx_i) 
d\bu 
                       \geq \rho -\xi_{i}, ~ i \in [m],
\int_{\bu \in U} w_{\bu}d\bu  = 1,
 \bw  \geq \boldsymbol{0},~ \rho \in \Real,
\end{align}
where $h_\bu \in H_{U}$ is a given base classifiers, $\nu \in [0,1]$ is a constant parameter, 
$\rho$ is a \emph{target margin} and $\xi_i$ is a
slack variable.
The dual problem
(\ref{align:LPBoostPrimal}) 
is given as follows.
\begin{align}\label{align:LPBoostDual}
\min\limits_{\gamma,\bd} \quad & \gamma
\\ \nonumber
\text{sub.to} \quad 
& \sum_{i=1}^my_id_i h_\bu(\bx_i) \leq \gamma,
\quad  \bu \in U,~~
0 \leq d_i \leq 1/\nu m ~ (i \in [m]),~~
 \sum_{i=1}^m d_{i}  = 1, ~ \gamma \in \Real.
\end{align}
The dual problem is categorized as a semi-infinite program (SIP), since
it contains possibly infinitely many constraints.
Note that the duality gap is zero since the problem
(\ref{align:LPBoostDual})
is convex and the optimum is finite \citep[][Theorem 2.2]{Shapiro09}.
Our approach is to approximately solve
the primal and the dual problems for $U=\{ \bu: \|\bu\|_\Hilbert \leq \Lambda\}$
by solving the sub-problems over a finite subset $U' \subset U$.
Such approach is called column generation in linear programming, 
which add a new constraint (column) to the dual problems and solve them
iteratively.
LPBoost~\citep{demiriz-etal:ml02} is a well known example of the approach
in the boosting setting. At each iteration, LPBoost chooses a hypothesis
$h_\bu$ so that $h_\bu$ maximally violates the constraints in the current
dual problem. This sub-problem is called weak learning in the boosting
setting and formulated as follows:
\begin{align}\label{align:WeakLearn_u}
\max_{\bu \in \Hilbert} \quad 
\sum_{i=1}^my_id_i
\max_{j \in [Q]}  \left\langle \bu,  \Phi\left(\bx_i^{(j)}\right)
                       \right\rangle \quad \text{sub.to} \quad \|\bu\|_\Hilbert^2 \leq \Lambda^2.
\end{align}
However, the problem (\ref{align:WeakLearn_u}) cannot be solved
 directly, since we have only access to $U$ though the associated kernel.
Fortunately, the optimal solution $\bu^*$ of the problem 
can be written as a linear combination of  the functions
$K(\bx_i, \cdot)$ because of the following representer theorem. 
\begin{theo}[Representer Theorem]\label{theo:represent}
The optimal solution $\bu^*$ of optimization problem (\ref{align:WeakLearn_u}) has the form of
$\bu^* = \sum_{i=1}^m\sum_{j=1}^Q \alpha_{ij}K(\bx_i^{(j)}, \cdot)$.
\end{theo}
The proof is shown in the supplemental material.

By Theorem~\ref{theo:represent}, we can design 
a weak learner by solving the following equivalent problem:
\begin{align}\label{align:WeakLearn}
&\min_{\balpha} 
- \sum_{p:y_p=+1}\hat{d}_p
\max_{j \in [Q]} \sum_{i=1}^m \sum_{k=1}^{Q} \alpha_{ik} K\left( \bx_i^{(k)}, \bx_p^{(j)}\right)
+ \sum_{q:y_q=-1}\hat{d}_q 
\max_{j \in [Q]} \sum_{i=1}^m \sum_{k=1}^{Q}\alpha_{ik} K\left(
                      \bx_i^{(k)}, \bx_q^{(j)} \right) 
\\ \nonumber
&\text{sub.to} \quad 
\sum_{i, k}^m\sum_{j, l}^{Q}\alpha_{ij}\alpha_{kl}K \left(\bx_i^{(j)}, \bx_k^{(l)} \right)  \leq \Lambda^2.
\end{align}
\section{Algorithm}
The optimization problem (\ref{align:WeakLearn}) is difference of
convex functions (DC) programming problem
and we can obtain local optimum $\epsilon$-approximately by 
using DC Algorithm~\citep{Tao1988}.
For the above optimization problem,
we replace $\max_{j \in [Q]}\sum_{i=1}^m\sum_{k=1}^{Q} \alpha_{ik} K\left( \bx_i^{(k)}, \bx_q^{(j)}
  \right)$ with $\lambda_q$, then we get the equivalent 
optimization problem as below.
\begin{align}\label{align:WeakLearn2}
\min_{\balpha, \boldsymbol{\lambda}} \quad& 
- \sum_{p:y_p=+1} \hat{d}_p
\max_{j \in [Q]} \sum_{i=1}^m \sum_{k=1}^{Q} \alpha_{ik} K\left( \bx_i^{(k)}, \bx_p^{(j)}\right)
+ \sum_{q:y_q=-1}\hat{d}_q \lambda_q 
\\ \nonumber
\text{sub.to} \quad&  \sum_{i=1}^m\sum_{k=1}^Q\alpha_{ik} K\left( \bx_i^{(k)}, \bx_q^{(j)}  \right) 
                      \leq \lambda_q 
                      ~(j \in [Q], \forall q:y_q=-1), ~~\\ \nonumber
&\sum_{i, k=1}^m\sum_{j, l=1}^{Q}\alpha_{ij}\alpha_{kl}K \left(\bx_i^{(j)}, \bx_k^{(l)} \right)  \leq \Lambda^2.
\end{align}

We show the pseudo code of LPBoost using column 
generation algorithm,
and our weak learning algorithm 
using DC programming in 
Algorithm~\ref{alg:LPBoost} and~\ref{alg:WeakLearn}, respectively. 
In Algorithm~\ref{alg:WeakLearn}, 
the optimization problem (\ref{align:subprob}) can be solved by
standard QP solver.
\begin{algorithm}
\caption{LPBoost with WeakLearn}
\label{alg:LPBoost}
\begin{enumerate}
 \item {\bf Input}: $S$, $\Lambda$, $\epsilon > 0$
 \item {\bf Initialize}: $\bd_0 \leftarrow (\frac{1}{m}, \ldots, \frac{1}{m})$
 \item {\bf For} $t=1,\dots, T$
\begin{enumerate}
 \item $h_t\leftarrow$  Run {\bf WeakLearn}($S$, $\bd_{t-1}$,
   $\Lambda$, $\epsilon$)
 \item Solve optimization problem:  
\begin{align}
\nonumber
& (\gamma, \bd_t) \leftarrow \arg\min\limits_{\gamma,\bd} \quad  \gamma \quad \quad \quad
\\ \nonumber
&\text{sub.to}~~ 
 \sum_{i=1}^my_id_i h_j(\bx_i) \leq \gamma
~~(j \in [t]), ~~
0 \leq d_i \leq 1/\nu m ~~(i \in [m]),  \sum_{i=1}^m d_i  = 1, ~ \gamma \in \Real.
\end{align}
\end{enumerate}
\item $\bw \leftarrow$ Lagrangian multipliers of solution of last optimization problem
\item $g \leftarrow \sum_{j=1}^T w_j h_j$
\item {\bf Output}: $\sign(g)$
\end{enumerate}
\end{algorithm}
\begin{algorithm}
\caption{WeakLearn by DC Algorithm}
\label{alg:WeakLearn}
\begin{enumerate}
 \item {\bf Input}:  $S$, 
    $\bd$, $\Lambda$, $\epsilon$ (convergence parameter)
 \item {\bf Initialize}: $\balpha_0 \in \Real^{m \times Q}$, $f_0 \leftarrow \infty$
 \item {\bf For} $t=1,\dots $
\begin{enumerate}
 \item $j^*_p \leftarrow \arg\max_{j \in [Q]} \sum_{i=1}^m \sum_{k=1}^{Q}
   \alpha_{(t-1, ik)} K\left( \bx_i^{(k)}, \bx_p^{(j)}\right)$
   ~$(\forall p:y_p=1)$
 \item Solve optimization problem: 
\begin{align}\label{align:subprob}
 f \leftarrow \min_{\balpha, \boldsymbol{\lambda}} \quad& 
- \sum_{p:y_p=+1} \hat{d}_p
\sum_{i=1}^m \sum_{k=1}^{Q} \alpha_{ik} K\left( \bx_i^{(k)}, \bx_p^{(j^*_p)}\right)
+ \sum_{q:y_q=-1}\hat{d}_q \lambda_q 
\\ \nonumber
\text{sub.to} \quad&  \sum_{i=1}^m\sum_{k=1}^Q \alpha_{ik} K\left( \bx_i^{(k)}, \bx_q^{(j)}  \right) 
                      \leq \lambda_q 
                      ~(j \in [Q], \forall q:y_q=-1),\\ \nonumber
&\sum_{i, k=1}^m\sum_{j, l=1}^{Q}\alpha_{ij}\alpha_{kl}K \left(\bx_i^{(j)}, \bx_k^{(l)} \right)  \leq \Lambda^2.
\end{align}
 $\balpha_t \leftarrow \balpha$, $f_t \leftarrow f$
 \item {\bf If} $f_{t-1} - f_{t} \leq \epsilon$, {\bf then break}.
\end{enumerate}
\item {\bf Output}: $h$ such that $h(\bz) =
  \max_j\sum_{i=1}^m\sum_{k=1}^Q\alpha_{(t, ik)}K(\bx_i^{(k)}, \bz^{(j)})$
\end{enumerate}
\end{algorithm}
It is interesting to note that, 
if we restrict $\alpha_{ij}$ to unit vector,
we can get the optimal $\alpha$ analytically in each weak learning step, 
and the final hypothesis obtained by LPBoost totally behaves 
like time-series shapelets. 

%% file: experiments.tex
\section{Experiments}
In the following experiments, we show that
our methods are practically effective
for time-series classification problem
as one of the applications.

\subsection{Classification accuracy for UCR datasets}
We use UCR datasets~\citep{UCRArchive}, that are often used as
benchmark datasets for time-series classification methods.
For simplicity, we use several binary classification datasets 
(of course, \Ourmethod is applicable to multi-class).
The detailed information of the datasets is described 
in the left side of Table~\ref{tab:acc1}.
In our experiments, we assume that patterns are subsequence of 
a time series of length $\ell$. 

We set the hyper-parameters as following:
The length $\ell$ of pattern was searched in
$\{0.1, 0.2, 0.3, 0.4\}\times L$, 
where $L$ is the length of each time-series dataset,
and $\nu = \{0.2, 0.15, 0.1\}$.
We found good $\ell$ and $\nu$ through a grid search via $5$-fold cross
validation.
We use the Gaussian kernel $K(\bx, \bx') = \exp(-\sigma\|\bx -
\bx'\|^2)$.
For each $\ell$, we choose $\sigma$ from $\{0.0001, 0.001, \ldots, 10000 \}$,
which maximizes the variance of the values of the kernel matrix.
We set $\Lambda = 1$, and the number of maximum iterations of LPBoost is $100$, and
the number of maximum iterations of DC Programming is $10$.

Unfortunately, the quadratical normalization constraints of the optimization 
problem~(\ref{align:WeakLearn2}) have highly computational costs.
Thus, in practice, we replace the constraint with $1$-norm
regularization: $\|\balpha\|_1
\leq \Lambda$ and solve the linear program. 
We expect to obtain sparse solutions
while that makes the decision space of $\alpha$ small.
As a LP solver for the optimization problem of the weak learner 
and LPBoost, we used the CPLEX software.

The results of classification accuracy are shown in the right side of Table~\ref{tab:acc1}.
We referred to the experimental result reported by~\citet{HouKZ16} 
with regard to the accuracies of the following three state-of-the-arts
algorithms, 
LTS~\citep{Grabocka:2014:LTS:2623330.2623613}, 
IG-SVM~\citep{Hills:2014:CTS:2597434.2597448},
and FLAG~\citep{HouKZ16}.
It is reported that the above algorithms are 
highly accurate in practice.
Particularly, LTS is known as one of the most accurate algorithm in practice, 
however, it highly and delicately depends on the input
hyper-parameters \citep[see, e.g.,][]{WistubaGS15}.
Many of other existing shapelet-based methods including them
have difficulty of adjusting hyper-parameters.
However, as shown in Table~\ref{tab:acc1}, 
our algorithm had competitive performance 
for used datasets with a little effort to parameter search.

\begin{table*}[ht!]
\centering
\caption{The detailed information of used datasets and Classification accuracies (\%). \label{tab:acc1}}
\begin{tabular}{|c |r|r|r ||r|r|r|r| } \hline
dataset & \# train &  \# test & length &IG-SVM & LTS &FLAG& \Ourmethod   \\ \hline
ECGFiveDays &23 & 861 & 136 & 99.0 & {\bf 100.0}& 92.0 & 99.5  \\ \hline
Gun-Point & 50 & 150 & 150 & {\bf 100.0} & 99.6 & 96.7 & 98.7\\ \hline
ItalyPower. & 67 & 1029 & 24& 93.7 & {\bf 95.8} &94.6 & 94.5   \\ \hline
MoteStrain & 20 & 1252 & 84 & 88.7 & {\bf 90.0} &88.8 & 81.8  \\ \hline
SonyAIBO. & 20 & 601 & 70 &  92.7 & 91.0 &92.9 & {\bf 93.2} \\ \hline
TwoLeadECG & 23 & 1139 & 82 & {\bf 100.0} & {\bf 100.0} & 99.0 &93.5  \\ \hline
\end{tabular}
\end{table*}

\subsection{Sparsity and visualization analysis}
It is said that shapelets-based hypotheses
have great visibility \citep[see, e.g.,][]{Ye:2009:TSS:1557019.1557122}.
Now, we show that the solution obtained by \Ourmethod has 
high sparsity and it induces visibility, 
despite our final hypothesis contains (non-linear) kernels.
Let us explain the sparsity using Gun-point dataset as an example. 
Now, we focus on the final hypothesis: 
$g = \sum_{t=1}^Tw_th_t$, where $h_t=\max_{j}\sum_{k=1}^m\sum_{l=1}^Q \alpha_{t,
  kl}K(\bx_k^{(l)}, \cdot^{(j)})$.
The number of final hypotheses $h_t$s such that $w_t \neq 0$
is $5$, and the number of non-zero elements $\alpha_{t, kl}$ of 
such $h_t$s is $26$ out of $34000$ ($0.6\%$). 
\begin{figure}
 \begin{center}
  \includegraphics[width=75mm]{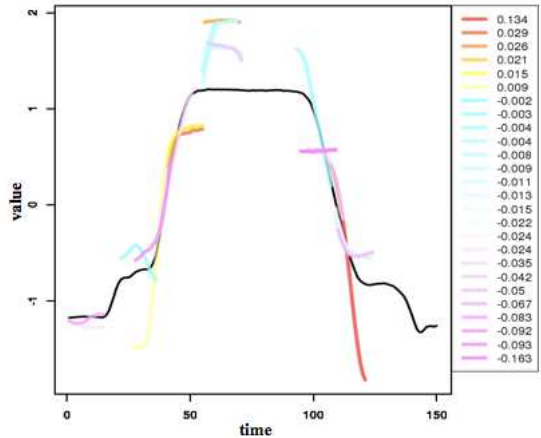}
 \caption{Visualization of discriminative pattern for an example of
   Gun-point data (negative label). 
 Black line is original time series.
 Positive value (red to yellow line) indicates the contribution rate
 for positive label,
 and negative value (light blue to purple line) indicates the
 contribution rate for negative label. \label{fig:gunpoint}}
 \end{center}
\end{figure}
Actually, for the other datasets, 
percentages of non-zero elements are from $0.08\%$ to $6\%$.
It is clear that the solution obtained by \Ourmethod has high sparsity.

Figure~\ref{fig:gunpoint} 
is an example of visualization
of a final hypothesis obtained by \Ourmethod for Gun-point time series
data.
The colored lines shows all of the 
$\bx_{k}^{(l)}$s in $g$ where both $w_t$ and $\alpha_{t, kl}$ are non-zero.
Each value of the legends show the multiplication of 
$w_t$ and $\alpha_{t, kl}$ corresponding to $\bx_{k}^{(l)}$.
Since it is hard to visualize the local features
over the Hilbert space,
we plotted each of them to match the original time series based on 
Euclidean distance.
In contrast to visualization analyses by 
time-series shapelets \citep[see, e.g.,][]{Ye:2009:TSS:1557019.1557122}, 
our visualization, colored lines and plotted position,
do not strictly represent the 
meaning of the final hypothesis because 
of the non-linear feature mapping.
However, we can say that colored lines represent
``discriminative pattern'' and certainly make important 
contributions to classification.
Thus, we believe that our solutions are useful for
domain experts to interpret important patterns.

%% file: conclusion.tex
\section{Conclusion}
\label{sec:conclusion}
In this paper, we show 
generalization error bounds of the hypothesis classes based on local features.
Further, we formulate an optimization problem, which fits in the
boosting framework. We design an efficient algorithm for the weak
learner using DC programming techniques.
Experimental results show that our method achieved competitive accuracy
with the existing methods with small parameter-tuning costs.

%% file: appendix.tex
\section{Supplemental material}
\subsection{Proof of Lemma 3}
\begin{proof} 
Since $U$ can be partitioned into
	$\cup_{\theta \in \Theta} \{ \bu \in U \mid \theta_\bu = \theta\}$,
\begin{align}
\label{eq:gauss1}
\nonumber
\GC_{S}(H_U) &= \frac{1}{m} \Expo_{\bsigma} \left[\sup_{\theta \in \Theta} 
\sup_{\bu \in U:\theta_{\bu}=\theta} \sum_{i=1}^m \sigma_i
\left\langle \bu, \Phi\left(\bx_i^{(\theta(i))}\right)
             \right\rangle\right] \\ \nonumber
&= \frac{1}{m} \Expo_{\bsigma} \left[\sup_{\theta \in \Theta} 
\sup_{\bu \in U:\theta_{\bu}=\theta} \left\langle \bu,
 \left(\sum_{i=1}^m \sigma_i \Phi\left(\bx_i^{(\theta(i))}\right)
  \right) \right \rangle \right] \\ \nonumber
&\leq \frac{1}{m} \Expo_{\bsigma} \left[\sup_{\theta \in \Theta} 
\sup_{\bu \in U} \left\langle \bu, 
\left(\sum_{i=1}^m \sigma_i \Phi\left(\bx_i^{(\theta(i))}\right)
  \right) \right \rangle \right] \leq \frac{\Lambda}{m} \Expo_{\bsigma} \left[
  \sup_{\theta \in \Theta} \left\| \sum_{i=1}^m \sigma_i
  \bx_i^{(\theta(i))} \right\|_\Hilbert
  \right]\\ \nonumber
  &= \frac{\Lambda}{m} \Expo_{\bsigma} \left[
  \sup_{\theta \in \Theta} \sqrt{\left\|\sum_{i=1}^m \sigma_i
  \Phi\left(\bx_i^{(\theta(i))}\right)\right\|_\Hilbert^2}
  \right] = \frac{\Lambda}{m} \Expo_{\bsigma} \left[
  \sqrt{\sup_{\theta \in \Theta} \left\|\sum_{i=1}^m \sigma_i
  \Phi\left(\bx_i^{(\theta(i))}\right)\right\|_\Hilbert^2}
  \right]\\ 
& \leq \frac{\Lambda}{m} \sqrt{ 
  \Expo_{\bsigma} \left[
  \sup_{\theta \in \Theta} \left\|\sum_{i=1}^m \sigma_i
  \Phi\left(\bx_i^{(\theta(i))}\right)\right\|_\Hilbert^2
  \right]}.
\end{align}
The first inequality is derived from the relaxation of $\bu$,
the second inequality is due to Cauchy-Schwarz inequality and
the fact $\|\bu\|_{\Hilbert} \leq \Lambda$, and
the last inequality is due to Jensen's inequality.
We denote by $\kernel^{(\theta)}$ the kernel matrix such that 
$\kernel_{ij}^{(\theta)} = \langle \Phi(\bx_i^{(\theta(i))}),
\Phi(\bx_j^{(\theta(j))}) \rangle$.
 Then, we have
 \begin{align}
  \Expo_{\bsigma} \left[
  \sup_{\theta \in \Theta} \left\|\sum_{i=1}^m \sigma_i
  \Phi\left(\bx_i^{(\theta(i))}\right)\right\|_\Hilbert^2
  \right]
  =
  \Expo_{\bsigma}\left[ \sup_{\theta \in \Theta}  
\sum_{i,j=1}^m \sigma_i\sigma_j 
\kernel_{ij}^{(\theta)}
\right].
 \end{align}
We now derive an upper bound of the r.h.s. as follows.
 
For any $c>0$, 
\begin{align*}
&\exp\left(
c \Expo_{\bsigma}\left[ \sup_{\theta \in \Theta}  
\sum_{i,j=1}^m \sigma_i\sigma_j 
\kernel_{ij}^{(\theta)}
\right]
\right) 
\leq 
\Expo_{\bsigma}\left[\exp\left(
c \sup_{\theta \in \Theta}  
\sum_{i,j=1}^m \sigma_i\sigma_j 
\kernel_{ij}^{(\theta)}
\right)
\right] \\
= &
\Expo_{\bsigma}\left[
\sup_{\theta \in \Theta}  
\exp\left(c
\sum_{i,j=1}^m \sigma_i\sigma_j 
\kernel_{ij}^{(\theta)}
\right)
\right] 
\leq 
\sum_{\theta \in \Theta}
\Expo_{\bsigma}\left[
\exp\left(c
\sum_{i,j=1}^m \sigma_i\sigma_j 
\kernel_{ij}^{(\theta)}
\right)
\right]
\end{align*}
The first inequality is due to Jensen's inequality, and
the second inequality is due to the fact that the supremum is
bounded by the sum.
By using the symmetry property of $\kernel^{(\theta)}$, we have
$\sum_{i,j=1}^m\sigma_i\sigma_j\kernel_{ij}^{(\theta)}=\T{\bsigma}\kernel^{(\theta)}\bsigma$,
which is rewritten as
\[
\T{\bsigma}\kernel^{(\theta)}\bsigma = \T{(\T{\Vmat}\bsigma)}
\left( \begin{array}{ccc}
\lambda_1 & \hfill & 0  \\
\hfill & \ddots & \hfill  \\
0 & \hfill & \lambda_m \\
\end{array} \right)
\T{\Vmat}\bsigma,
\]
where
$\lambda_1\geq \dots \geq \lambda_m\geq 0$ are the
eigenvalues of $\kernel^{(\theta)}$ and 
$\Vmat = (\bv_1, \ldots, \bv_m)$
is the orthonormal matrix such that $\bv_i$ is the eigenvector
that corresponds to the eigenvalue $\lambda_i$.
By the reproductive property of Gaussian distribution,
$\T{\Vmat} \bsigma$ obeys the same Gaussian distribution as well.
So,
\begin{align*}
&\sum_{\theta \in \Theta}
\Expo_{\bsigma}\left[
\exp\left(c
\sum_{i,j=1}^m \sigma_i\sigma_j 
\kernel_{ij}^{(\theta)}
\right)
\right] 
=
\sum_{\theta \in \Theta}
\Expo_{\bsigma}\left[
\exp\left(c
\T{\bsigma}\kernel^{(\theta)}\bsigma
\right)
\right] \\
= &
\sum_{\theta \in \Theta}
\Expo_{\bsigma}\left[
\exp\left(c
\sum_{k=1}^m\lambda_k(\T{\bv_k}\bsigma)^2
\right)
\right] 
= 
\sum_{\theta \in \Theta}
\Pi_{k=1}^m
\Expo_{\sigma_k}\left[
\exp\left(c
\lambda_k\sigma_k^2
\right)
\right] ~~(\text{replace}~\bsigma = \T{\bv_k}\bsigma)\\
 = &
 \sum_{\theta \in \Theta}
\Pi_{k=1}^m
 \left(
 \int_{-\infty}^{\infty}
 \exp\left(
 c \lambda_k\sigma^2
 \right)
  \frac{\exp(-\sigma)^2}{\sqrt{2\pi}}d\sigma
\right)
 = 
\sum_{\theta \in \Theta}
\Pi_{k=1}^m
 \left(
  \int_{-\infty}^{\infty}
  \frac{\exp(-(1-c\lambda_k)\sigma^2)}{\sqrt{2\pi}}d\sigma\right).
\end{align*}
Now we replace $\sigma$ by $\sigma' =
\sqrt{1-c\lambda_k}\sigma$. Since
$d\sigma'=\sqrt{1-c\lambda_k}d\sigma$,
we have:
\begin{align*}
\int_{-\infty}^{\infty}
\frac{\exp(-(1-c\lambda_k)\sigma^2)}{\sqrt{2}\pi}d\sigma
= \frac{1}{\sqrt{2\pi}} \int_{-\infty}^{\infty} 
\frac{\exp(-\sigma'^2)}{\sqrt{1-c\lambda_k}}d\sigma'
= \frac{1}{\sqrt{1-c\lambda_k}}.
\end{align*}
 Now, by letting $c=\frac{1}{2 \max_{i, \theta}\lambda_i}=1/(2\lambda_1)$ and
 applying the inequality that $\frac{1}{(1-x)} \leq
 1+2(\sqrt{2}-1)x$ for  $0 \leq x \leq \frac{1}{2}$, 
 the bound becomes
 \begin{align}
 &\exp\left(
c \Expo_{\bsigma}\left[ \sup_{\theta \in \Theta}  
\sum_{i,j=1}^m \sigma_i\sigma_j 
\kernel_{ij}^{(\theta)}
\right]
  \right)
  \leq
\sum_{\theta \in \Theta}
\Pi_{k=1}^m
  \left(
1+2(\sqrt{2}-1)c\lambda_k
  \right)
 \end{align}
because it holds that $\frac{1}{(1-x)} \leq 1+2(\sqrt{2}-1)x$ in $0 \leq x \leq \frac{1}{2}$.
Further, taking logarithm, dividing the both sides by $c$ and applying
 $\ln(1+x) \leq x$, we get:
\begin{align}
\label{eq:gauss2}
\Expo_{\bsigma}\left[ \sup_{\theta \in \Theta}  
\sum_{i,j=1}^m \sigma_i\sigma_j 
\kernel_{ij}^{(\theta)}
\right]
&\leq
 (\sqrt{2}-1)\sum_{k=1}^m \lambda_k  + 2\lambda_1 \ln |\Theta|\\
 &=
 (\sqrt{2}-1)\tr(\kernel) + 2\lambda_1 \ln |\Theta|\\
 &\leq
 (\sqrt{2}-1)m R^2 + 2 mR^2 \ln |\Theta|,
\end{align}
 where the last inequality holds since $\lambda_1=\|\kernel\|_2 \leq m
 \|\kernel\|_{\max} \leq R^2$.
By equation~(\ref{eq:gauss1}) and (\ref{eq:gauss2}), 
we have:
\begin{align}
\GC_S(H) \leq
\frac{\Lambda}{m}
\sqrt{
\Expo_{\bsigma}
\left[\sup_{\theta \in \Theta}  
\sum_{i,j=1}^m \sigma_i\sigma_j 
\kernel_{ij}^{(\theta)}
\right]} 
\leq 
\frac{\Lambda R\sqrt{(\sqrt{2}-1) + 2\ln 
|\Theta|
}}{\sqrt{m}}.
\end{align}
\end{proof}

\subsection{Proof of Lemma 4}
\begin{proof}
 We will reduce the problem of counting intersections of the Voronoi
 diagrams to that of counting possible labelings by hyperplanes for
 some set.
 Note that for each neighboring Voronoi regions, the border is a part of
 hyperplane since the closeness is defined in terms of the inner
 product.
 So, by simply extending each border to a hyperplane, we obtain
 intersections of halfspaces defined by the extended hyperplanes.
 Note that, the size of these intersections gives an upper bound of
 intersections of the Voronoi diagrams.
 More precisely, we draw hyperplanes for each pair of points in
 $\calP_S=\{\Phi(\bx_i^{(j)})\ \mid i \in [m], j \in [Q]\}$ so that each point on
 the hyperplane has the same inner product between two points.
 Note that for each pair $\bz, \bz' \in \calP_S$, the normal vector of
 the hyperplane is given as $\bz -\bz'$ (by fixing the sign arbitrary).
 So, the set of hyperplanes obtained by this procedure is exactly
 $\pdiff$. The size of $\pdiff$ is ${mQ} \choose 2$, which is at most $m^2Q^2$.
 Now, we consider a ``dual'' space by viewing each hyperplane as a point and each point in $U$ as a
 hyperplane.
 Note that points $\bu$ (hyperplanes in the dual) in an intersection give the same labeling on the
 points in the dual domain.
 Therefore, the number of intersections in the original domain is the
 same as the number of the possible labelings on $\pdiff$ by hyperplanes
 in $U$. By the classical Sauer's
 Lemma and the VC dimension of hyperplanes~(see, e.g., Theorem 5.5
 in~\cite{LK02}), the size is at most $O((m^2Q^2)^{d_{\Phi,S}^*})$.
\end{proof}

\subsection{Proof of Theorem 2}
\begin{proof}
 (i) In this case, the Hilbert space $\Hilbert$is contained in
 $\Real^\ell$. Then, by the fact that VC dimension $d_{\Phi,S}^*$ is at
 most $\ell$ and Lemma 4, the statement holds.

 (ii)
 f $\left\langle\Phi(\bz), \Phi\bx \right\rangle$ is monotone
decreasing for $\|\bz-\bx\|$,
then the following holds:
\[
\arg\max_{\bx \in \calX} \left\langle\Phi(\bz), \Phi(\bx) \right\rangle = \arg\min_{\bx \in \calX} \|\bz-\bx\|_2.
\]
Therefore, $\max_{\bu:\|\bu\|_\Hilbert = 1}\langle \bu, \Phi(\bx)
\rangle = \|\Phi(\bx)\|_\Hilbert$, where $\bu =
\frac{\Phi(\bx)}{\|\Phi(\bx)\|_\Hilbert}$. 
It indicates that the number of Voronoi cells
made by $V_i^{(j)}=\{\bz \in \Real^\ell \mid j=\arg\max_{k \in [Q]}
  \bz \cdot \bx_i^{(k)} \}$ corresponds to the
$\hat{V}_i^{(j)}=\{\Phi(\bz) \in \Hilbert \mid j=\arg\max_{k \in [Q]}
 \langle \Phi(\bz), \Phi(\bx_i^{(k)}) \rangle \}$.
 Then, by following the same argument for the linear kernel case, we get
 the same statement.

 (iii)
 We follow the argument in Lemma 4.
 For the set of classifiers $F=\{f:\pdiff \to \{-1,1\} \mid
 f=\sign(\langle \bu, \bz\rangle), \|\bu|_\Hilbert, 
 \min_{\bz in \pdiff}| \langle \bu, \bz\rangle|=\mu \}$,
 its VC dimension is known to be at most $R\Lambda/\mu$
 for $\pdiff \subseteq \{\bz \mid \|\bz\|_\Hilbert \leq R\}$ (see, e.g., \cite{LK02}).
 By the definition of $\mu^*$,
 for each intersections given by hyperplanes, there always exists a point $\bu$
 whose inner product between each hyperplane is at least $\mu^*$.
 Therefore, the size of the intersections is bounded by the number of possible
 labelings in the dual space by $U''=\{\bu \in \Hilbert,
 \|\bu\|_\Hilbert \leq \Lambda,  \min_{\bz in \pdiff}| \langle \bu,
 \bz\rangle|=\mu^* \}\}$.
 Thus we obtain that $d_{\Phi, S}^*$ is at most $R\Lambda /\mu^*$ and by
 Lemma 4, we complete the proof.
\end{proof}

\subsection{Proof of Theorem 3}
 \begin{proof}
We can rewrite the optimization problem (4)
by using $\theta \in \Theta$ defined in Subsection 3.2 as
  follows:
\begin{align}
 \max_{\theta \in \Theta}\max_{\bu \in \Hilbert: \theta_\bu=\theta} \quad& 
\sum_{i=1}^my_id_i
\left\langle \bu,  \Phi\left(\bx_i^{(\theta(i))}\right)  \right\rangle \\
 \nonumber
\text{sub.to} \quad& \|\bu\|_\Hilbert^2 \leq \Lambda^2.
\end{align}
Thus, if we fix $\theta \in \Theta$, we have a sub-problem. Since
  the constraint $\theta=\theta_\bu$ can be written as linear
  constraints, each sub-problem is equivalent to a convex optimization.  
Indeed, each sub-problem can be written as the equivalent unconstrained
  minimization (by neglecting constants in the objective)
\[
\min_{\bu \in \Hilbert} \beta \|\bu\|^2_\Hilbert -\sum_{i=1}^m\sum_{j=1}^Q
 \lambda_{i,j} \left\langle \bu, \Phi\left(\bx_i^{(\theta(i))}\right)  \right\rangle
  -\sum_{i=1}^my_id_i
 \left\langle \bu,  \Phi\left(\bx_i^{(\theta(i))}\right)  \right\rangle,
\]
 where $\beta$ and $\lambda_{i,j}$ $(i \in [m], j \in [Q])$ are the corresponding positive constants.
Now for each sub-problem, we can apply the standard Representer Theorem
  argument (see, e.g., \cite{Mohri.et.al_FML}).
  Let $\Hilbert_1$ be the subspace $\{\bu \in \Hilbert \mid
 \bu=\sum_{i=1}^m\sum_{j=1}^Q \alpha_{ij}\Phi(x_i^{(j)}), \alpha_{ij}\in \Real\}$.
 We denote $\bu_1$ as the orthogonal projection of $\bu$ onto
 $\Hilbert_1$ and any $\bu\in \Hilbert$ has the decomposition
 $\bu=\bu_1 + \bu^{\perp}$. Since $\bu^{\perp}$ is orthogonal
 w.r.t. $\Hilbert_1$, $\|\bu\|_\Hilbert^2=\|\bu_1\|_\Hilbert^2 +
 \|\bu^{\perp}\|_\Hilbert^2 \geq \|\bu_1\|_\Hilbert^2$.
On the other hand, $\left\langle \bu,  \Phi\left(\bx_i^{(j)}\right)
 \right\rangle =\left\langle \bu_1,  \Phi\left(\bx_i^{(j)}\right)  \right\rangle$.
 Therefore, the optimal solution of each sub-problem has to be contained
  in $\Hilbert_1$.
  This implies that the optimal solution, which is the maximum over all
  solutions of sub-problems, is contained in $\Hilbert_1$ as well.
 \end{proof}